\documentclass[letterpaper, 10 pt, conference]{ieeeconf}

\usepackage{times}
\makeatletter
\let\NAT@parse\undefined
\makeatother
\usepackage[numbers]{natbib}
\usepackage[hidelinks,bookmarks=true]{hyperref}
\usepackage{amsmath}
\usepackage{amssymb}

\usepackage{amsthm}
\usepackage{xcolor}
\usepackage{graphicx}
\usepackage{booktabs}
\usepackage{bm}
\usepackage{mathtools}
\newtheorem{theorem}{Theorem}
\newtheorem{lemma}[theorem]{Lemma}
\newtheorem{corollary}[theorem]{Corollary}
\newtheorem{proposition}[theorem]{Proposition}
\newtheorem{definition}[theorem]{Definition}

\DeclareMathOperator{\tr}{tr}

\DeclareMathOperator*{\argmin}{arg\,min}

\newcommand{\reals}{\mathbb{R}}
\newcommand{\expect}{\mathbb{E}}
\newcommand{\prob}{\mathbb{P}}
\newcommand{\xvec}{x}
\newcommand{\evec}{e}
\newcommand{\xivec}{\xi}
\newcommand{\Amat}{A}
\newcommand{\Bmat}{B}
\newcommand{\Cmat}{C}
\newcommand{\opnorm}[1]{\left\|#1\right\|_{\mathrm{op}}}
\newcommand{\twonorm}[1]{\left\|#1\right\|_2}
\newcommand{\psd}{\preceq}

\usepackage{flushend}
\providecommand{\IEEEkeywordsname}{Index Terms}
\newenvironment{IEEEkeywords}{\par\addvspace{0.5\baselineskip}\noindent\small\textbf{\textit{\IEEEkeywordsname.\ }}}{\par\addvspace{0.5\baselineskip}}

\begin{document}

\title{Behavior Cloning Under PD Control:\\A Finite-Horizon Theory of Gain-Dependent Error Amplification}

\IEEEoverridecommandlockouts
\author{Junghoon Seo%
\thanks{Junghoon Seo is with PIT IN Corp., South Korea (e-mail: sjh@pitin-ev.com).}%
}

\maketitle

\begin{abstract}
Behavior cloning (BC) on position-controlled robots is shaped by the PD loop that executes policy actions. We give a finite-horizon, nonasymptotic analysis of how controller gains affect BC failure. Independent sub-Gaussian action errors propagate through gain-dependent closed-loop dynamics into sub-Gaussian position errors. The resulting failure tail is controlled by controller amplification multiplied by validation loss and generalization slack, so validation loss alone can mis-rank gains. Under shape-preserving \emph{upper-bound} assumptions, the analysis separates label difficulty, injection strength, and contraction, ranking compliant-overdamped gains as tightest and stiff-underdamped gains as loosest, with the mixed regimes system-dependent. In the canonical scalar second-order PD system, stationary position-error variance increases with stiffness and decreases with damping over the stable range, and exact zero-order-hold discretization inherits the ordering to leading order. This extends the error-attenuation explanation of~\cite{bronars2026tune} to finite-horizon failure bounds.
\end{abstract}


\section{Introduction}\label{sec:intro}

Position-controlled action spaces are a standard interface for learned manipulation policies~\cite{chi2025diffusion, chi2024universal, zhao2023learning}. A learned policy outputs a position target, which a PD controller~\cite{kelly1997pd} converts to torques,
\begin{equation}\label{eq:pd}
\tau_t = K_p(a_t - q_t) - K_d\,\dot q_t + g(q_t),
\end{equation}
where $K_p,K_d$ are diagonal positive-definite proportional and derivative gains. Although this interface is ubiquitous, gain selection for \emph{learning} remains largely heuristic. For behavior cloning (BC), this is problematic: gains alter the labels learned from demonstrations, map residual action errors into robot motion, and set how quickly those errors are dissipated.

Bronars~et~al.~\cite{bronars2026tune} find that BC benefits from \emph{compliant} low-$K_p$ and \emph{overdamped} high-$K_d$ gains. These settings improve closed-loop success despite higher validation loss, which they explain by PD error attenuation and verify with noise-injection experiments. Missing is a finite-horizon theory connecting this mechanism to validation loss, generalization slack, and failure probabilities. The gap matters because validation loss is measured before the controller acts, while task failure occurs after residual errors have passed through gain-dependent closed-loop dynamics.

Existing continuous-action BC theory studies function-approximation sensitivity~\cite{pmlr-v291-simchowitz25a} and autoregressive noise amplification~\cite{block2024butterfly}, and classical compounding-error analyses~\cite{ross2010efficient, ross2011reduction, spencer2021feedback} treat the dynamics as a black box. Action-space and controller-design studies show that the interface matters~\cite{mandlekar2021what, aljalbout2024role, esser2024action}, while stochastic-control analyses propagate disturbances through linear systems~\cite{recht2019tour, boffi2021regret, tu2019gap}. We instead expose the gain-tuned closed loop and show that closed-loop BC performance is governed by prediction loss multiplied by controller amplification.

Our contributions are:
\begin{enumerate}
\item \textbf{Sub-Gaussian propagation} (Theorem~\ref{thm:subgaussian}). Independent action errors induce position-error tails through a finite-horizon Lyapunov sum whose weights depend on the PD gains.
\item \textbf{Failure bound} (Theorem~\ref{thm:failure}). Rollout failure is controlled by validation loss and generalization slack multiplied by closed-loop amplification.
\item \textbf{Scalar ordering} (Theorems~\ref{thm:scalar}--\ref{thm:regime}). Shape-preserving upper bounds reduce regime comparison to a single scalar index.
\item \textbf{Canonical monotonicity} (Theorem~\ref{thm:canonical}). Scalar position-error variance increases with stiffness and decreases with damping, and exact ZOH discretization is consistent to leading order.
\end{enumerate}

\section{Related Works}\label{sec:related}

\paragraph{Behavior cloning theory}
Behavior cloning~\cite{pomerleau1989alvinn, 10.5555/647636.733043} learns a policy by supervised regression on expert state-action pairs. The foundational analysis of Ross and Bagnell~\cite{ross2010efficient} shows that the expected cost of the learned policy grows quadratically with horizon~$T$ from compounding errors, motivating interactive approaches such as DAgger~\cite{ross2011reduction}. Subsequent work refines these bounds under structural assumptions~\cite{spencer2021feedback, florence2022implicit}, sharpens the horizon dependence~\cite{foster2024behavior}, and links low-level dynamical stability to task-level guarantees~\cite{block2023provable}. Building on the controller-level mechanism identified in~\cite{bronars2026tune}, our analysis treats the controller gain as a structural variable in nonasymptotic BC error propagation.

\paragraph{Impedance control}
The interaction between compliance and task performance has a long history in robotics~\cite{hogan1984impedance, hogan1985impedance, khatib2003unified}, with variable impedance learning~\cite{bogdanovic2020learning, wu2021framework, kronander2013learning} adapting compliance during execution. Our contribution is orthogonal. We analyze how \emph{fixed} gains modulate the propagation of prediction errors during BC rollout.

\paragraph{Controller design for robot learning}
The choice of action space and controller gains has long been recognized as a critical design axis for robot learning~\cite{aljalbout2024role, esser2024action, mandlekar2021what}. Kim~et~al.~\cite{kim2023torque} show that torque-level control improves task-agnostic transfer, and Bronars~et~al.~\cite{bronars2026tune} provide a systematic study of how PD gains affect BC, RL, and sim-to-real transfer, identifying compliant-overdamped settings as the BC optimum. They also propose and test a mechanistic explanation based on closed-loop error attenuation, using open-loop noise-injection experiments to isolate the effect and deriving the stationary variance that captures it. The theoretical analysis below extends that account to sub-Gaussian propagation, finite-horizon failure bounds, and generalization slack.

\section{Problem Formulation}\label{sec:problem}

\subsection{Gain-Dependent Closed-Loop Error Dynamics}\label{sec:dynamics}

Consider a robot with~$n$ joints and inertia~$M\in\reals^{n\times n}$ operating under the PD controller~\eqref{eq:pd} with gain setting~$K=(K_p,K_d)$. Let~$q_t^\star$ denote the expert joint trajectory and~$a_t^\star$ the expert action. Define the position error $\evec_t := q_t-q_t^\star$ and the prediction error $\xivec_t := \hat\pi_K(s_t)-a_t^\star$. Substituting $a_t = a_t^\star+\xivec_t$ into~\eqref{eq:pd} and linearizing about the expert trajectory yields
\begin{equation}\label{eq:error-ct}
M\,\ddot\evec_t = -K_p\,\evec_t - K_d\,\dot\evec_t + K_p\,\xivec_t,
\end{equation}
so the action error enters scaled by~$K_p$ and stiffer gains amplify both the restoring force on $\evec_t$ and the injection of action errors. Setting $\xvec_t := [\evec_t^\top,\,\dot\evec_t^\top]^\top \in \reals^{2n}$ gives the continuous state-space form
\begin{equation}\label{eq:cont-state}
\dot\xvec_t = \Amat_K^{\mathrm{c}}\xvec_t + \Bmat_K^{\mathrm{c}}\xivec_t,\qquad \evec_t = \Cmat\xvec_t,
\end{equation}
with $\Amat_K^{\mathrm{c}} = \begin{bsmallmatrix} 0 & I_n \\ -M^{-1}K_p & -M^{-1}K_d \end{bsmallmatrix}$, $\Bmat_K^{\mathrm{c}} = \begin{bsmallmatrix} 0 \\ M^{-1}K_p \end{bsmallmatrix}$, and $\Cmat = [I_n,\,0]$.

The policy emits a new target every~$\Delta t$ seconds and the controller holds it constant between updates, so the sampled state obeys
\begin{equation}\label{eq:dynamics}
\xvec_{t+1} = \Amat_K\xvec_t + \Bmat_K\xivec_t,\qquad \evec_t = \Cmat\xvec_t,
\end{equation}
with matrices given by the \emph{exact zero-order-hold (ZOH)} sampling of~\eqref{eq:cont-state},
\begin{equation}\label{eq:zoh}
\Amat_K = e^{\Amat_K^{\mathrm{c}}\Delta t},\quad \Bmat_K = \Bigg(\int_0^{\Delta t}\! e^{\Amat_K^{\mathrm{c}} s}\,ds\Bigg)\Bmat_K^{\mathrm{c}}.
\end{equation}
We consider gains for which the sampled closed loop is asymptotically stable, $\rho(\Amat_K)<1$.

\subsection{Action Prediction Error Model}\label{sec:error-model}

We model the prediction errors $\{\xivec_t\}$ as \emph{independent, mean-zero, sub-Gaussian} random vectors. There exists a PSD matrix $\Sigma_K^{\mathrm{roll}}\succeq 0$ such that
\begin{equation}\label{eq:subgauss}
\expect\!\left[\exp(\lambda^\top\xivec_t)\right] \le \exp\!\left(\tfrac{1}{2}\lambda^\top\Sigma_K^{\mathrm{roll}}\lambda\right),\quad \forall\,\lambda\in\reals^m.
\end{equation}
The proxy covariance~$\Sigma_K^{\mathrm{roll}}$ governs the tail behavior, with the subscript~$K$ reflecting that different gains induce different action labels and hence different residual statistics. The \emph{rollout-local population action MSE} is $L_{\mathrm{roll}}(K) := \expect[\twonorm{\xivec_t}^2] = \tr(\Sigma_K^{\mathrm{roll}})$.

These independence and tail assumptions are rollout-local. They are most appropriate for high-rate control with locally linear motion and residuals dominated by fresh light-tailed noise (sensing, quantization, communication, numerical, or stochastic inference) rather than persistent model bias. If low-frequency bias, contact modes, or accumulated covariate shift induce temporal dependence, the results remain conservative after replacing~$\Sigma_K^{\mathrm{roll}}$ by a uniform conditional sub-Gaussian proxy, or inflating it to absorb the residual correlation time.

\subsection{Task Success, Validation Loss, and Amplification}\label{sec:fail-and-amp}

A rollout is successful if the position error stays within a success tube of radius~$r>0$ over horizon~$T$. The complementary failure event is
\begin{equation}\label{eq:fail}
\mathrm{Fail}_T := \left\{\exists\, t\le T : \twonorm{\evec_t}\ge r\right\}.
\end{equation}
After training, the empirical validation loss over $N_{\mathrm{va}}$ held-out pairs $\{(\tilde s_j,\tilde a_j)\}$ is $\hat L_{\mathrm{va}}(K) := N_{\mathrm{va}}^{-1}\sum_j\twonorm{\hat\pi_K(\tilde s_j)-\tilde a_j}^2$, and the training loss $\hat L_{\mathrm{tr}}(K)$ is defined analogously. A standard generalization argument gives, with probability at least~$1-\delta$,
\begin{equation}\label{eq:gen}
L_{\mathrm{roll}}(K) \le \hat L_{\mathrm{va}}(K) + \varepsilon_{\mathrm{gen}}(K,\delta).
\end{equation}
We define the \emph{finite-horizon amplification index}
\begin{equation}\label{eq:gamma}
\Gamma_T(K) := \max_{0\le t\le T}\sum_{s=0}^{t-1}\opnorm{\Cmat\Amat_K^s\Bmat_K}^2,
\end{equation}
which measures the worst-case cumulative operator-norm gain from action errors to position errors and depends on the gains through $\Amat_K$ and $\Bmat_K$.

\section{Main Results}\label{sec:results}

We present four results in a logical chain: a sub-Gaussian propagation theorem, a closed-loop failure bound, a scalar ordering reduction in the upper-bound direction, and a regime comparison.

\subsection{Sub-Gaussian Error Propagation}

Theorem~\ref{thm:subgaussian} lifts the per-step sub-Gaussian assumption on the action error to a horizon-wide statement about the position error. The lift is a direct consequence of linearity: a weighted sum of independent sub-Gaussians is sub-Gaussian, and the weights are exactly the closed-loop impulse response at lag~$s$. The Lyapunov limit~$S_\infty(K)$ is the standard quantity for disturbance-to-state propagation in discrete LTI systems, and the role of this paper is to track how every ingredient depends on the controller gain, so tuning $(K_p, K_d)$ reshapes the entire tail of~$\evec_t$.

\begin{theorem}[Sub-Gaussian propagation]\label{thm:subgaussian}
Consider~\eqref{eq:dynamics} with~$\xvec_0 = 0$ and action errors satisfying~\eqref{eq:subgauss}. Define the \emph{proxy matrix}
\begin{equation}\label{eq:St}
S_t(K) := \sum_{s=0}^{t-1} \Amat_K^s \Bmat_K \Sigma_K^{\mathrm{roll}} \Bmat_K^\top (\Amat_K^\top)^s,
\end{equation}
and its position-error projection $X_t(K) := \Cmat\, S_t(K)\, \Cmat^\top$. Then~$\evec_t$ is mean-zero sub-Gaussian with proxy~$X_t(K)$. For all~$u \in \reals^n$,
\begin{equation}\label{eq:et-subgauss}
\expect\!\left[\exp(u^\top \evec_t)\right] \le \exp\!\left(\tfrac{1}{2} u^\top X_t(K)\, u\right),
\end{equation}
\begin{equation}\label{eq:et-tail}
\prob\!\left(|u^\top \evec_t| \ge r\right) \le 2\exp\!\left(-\frac{r^2}{2\, u^\top X_t(K)\, u}\right).
\end{equation}
If~$\rho(\Amat_K) < 1$, then~$S_t(K) \to S_\infty(K)$, the unique PSD solution of the discrete Lyapunov equation
\begin{equation}\label{eq:lyap}
S_\infty(K) = \Amat_K S_\infty(K) \Amat_K^\top + \Bmat_K \Sigma_K^{\mathrm{roll}} \Bmat_K^\top,
\end{equation}
and the stationary proxy is~$X_\infty(K) = \Cmat\, S_\infty(K)\, \Cmat^\top$.
\end{theorem}

\begin{proof}
Unrolling~\eqref{eq:dynamics} from~$\xvec_0=0$ gives $\evec_t = \sum_{s=0}^{t-1}\Cmat\Amat_K^s\Bmat_K\xivec_{t-1-s}$. For $u\in\reals^n$ and $h_s := \Bmat_K^\top(\Amat_K^\top)^s\Cmat^\top u$, we have $u^\top\evec_t = \sum_s h_s^\top\xivec_{t-1-s}$. By independence and the sub-Gaussian MGF bound,
\begin{align}
\expect\!\left[e^{u^\top\evec_t}\right]
&\le \prod_s \exp\!\left(\tfrac{1}{2}h_s^\top\Sigma_K^{\mathrm{roll}}h_s\right)\nonumber\\
&= \exp\!\left(\tfrac{1}{2}u^\top X_t(K)u\right).
\end{align}
The tail bound follows from the Chernoff method, and Schur stability with the discrete Lyapunov theorem yields convergence to $X_\infty(K)$.
\end{proof}

\begin{corollary}[Euclidean tail bound]\label{cor:euclidean}
Under Theorem~\ref{thm:subgaussian},
\begin{equation}\label{eq:euclidean-tail}
\prob\!\left(\twonorm{\evec_t}\ge r\right) \le 2n\,\exp\!\left(-\frac{r^2}{2n\,\lambda_{\max}(X_t(K))}\right).
\end{equation}
\end{corollary}

\begin{proof}
Diagonalize $X_t = Q\Lambda Q^\top$ and set $y = Q^\top\evec_t$. If $\twonorm{y}\ge r$ then some $|y_i|\ge r/\sqrt n$, and each $y_i$ is sub-Gaussian with proxy $\lambda_i\le\lambda_{\max}(X_t)$. A union bound over coordinates gives the result.
\end{proof}

The dimensional factor~$n$ enters twice, once as a linear prefactor and once inside the denominator of the exponent. In the scalar case $n=1$ both factors disappear and the bound reduces to the one-dimensional Chernoff tail. For robot manipulators $n$ is typically between six and twelve, and the corresponding slack is mild on a logarithmic scale.

\subsection{Closed-Loop Failure Bound}

\begin{theorem}[Failure bound via validation loss]\label{thm:failure}
Under the framework of Section~\ref{sec:problem}, with probability at least~$1 - \delta$ over training randomness,
\begin{multline}\label{eq:fail-bound}
\prob\!\left(\mathrm{Fail}_T \mid \hat\pi_K\right) \le \\
2n(T{+}1)\exp\!\left(-\frac{r^2}{2n\,\Gamma_T(K)\!\left(\hat L_{\mathrm{va}}(K) + \varepsilon_{\mathrm{gen}}(K,\delta)\right)}\right).
\end{multline}
\end{theorem}

\begin{proof}
From Theorem~\ref{thm:subgaussian},
$\lambda_{\max}(X_t(K)) \le \sum_s\opnorm{\Cmat\Amat_K^s\Bmat_K}^2\,\lambda_{\max}(\Sigma_K^{\mathrm{roll}}) \le \Gamma_T(K)\,L_{\mathrm{roll}}(K)$, where the second inequality uses $\lambda_{\max}(\Sigma_K^{\mathrm{roll}})\le\tr(\Sigma_K^{\mathrm{roll}}) = L_{\mathrm{roll}}(K)$. The generalization bound~\eqref{eq:gen} replaces $L_{\mathrm{roll}}(K)$ by $\hat L_{\mathrm{va}}(K)+\varepsilon_{\mathrm{gen}}(K,\delta)$ with probability~$1-\delta$. Substituting into~\eqref{eq:euclidean-tail} and union-bounding over $t=0,1,\dots,T$ yields~\eqref{eq:fail-bound}.
\end{proof}

The exponent of~\eqref{eq:fail-bound} is governed by the product $\Gamma_T(K)\,(\hat L_{\mathrm{va}}(K)+\varepsilon_{\mathrm{gen}}(K,\delta))$ rather than by the loss alone, so a gain $K_1$ with lower training or validation loss than $K_2$ can incur a strictly worse failure bound once its amplification is correspondingly larger. This matches the attenuation-difficulty tradeoff of~\cite{bronars2026tune}, in which compliant-overdamped gains attain higher closed-loop success despite higher validation MSE because their amplification is small enough to dominate the product.

\subsection{Scalar Ordering Reduction}

To compare gain regimes through a single scalar, we impose structural conditions on the error dynamics that yield an \emph{upper bound} on the proxy matrix. Every inequality below is an upper bound, and reversing any one invalidates the chain.

\begin{definition}[Shape-preserving upper-bound structure]\label{def:shape}
The error dynamics~\eqref{eq:dynamics} have \emph{shape-preserving upper-bound structure} if there exist fixed PSD matrices $\bar W,\bar\Sigma\succeq 0$ and scalar functions $l(K),b(K)>0$, $\rho_*(K)\in[0,1)$ such that
\begin{align*}
&\text{(i)\,label difficulty:}      &\Sigma_K^{\mathrm{roll}} &\psd l(K)\,\bar\Sigma,\\
&\text{(ii)\,injection:}             &\Bmat_K\bar\Sigma\Bmat_K^\top &\psd b(K)\,\bar W,\\
&\text{(iii)\,Lyapunov contraction:} &\Amat_K\bar W\Amat_K^\top &\psd \rho_*(K)^2\,\bar W.
\end{align*}
\end{definition}

Condition~(i) bounds the rollout-error proxy at gain~$K$, condition~(ii) bounds how strongly action errors are injected into the state space, and condition~(iii) certifies $\bar W$ as a Lyapunov function for $\Amat_K$ with rate $\rho_*(K)^2$.

\begin{theorem}[Scalar ordering reduction]\label{thm:scalar}
Under Definition~\ref{def:shape}, define the \emph{ordering index}
\begin{equation}\label{eq:psi}
\Psi(K) := \frac{b(K)\, l(K)}{1 - \rho_*(K)^2}.
\end{equation}
Then
\begin{equation}\label{eq:Xinfty-bound}
X_\infty(K) \psd \Psi(K)\, \bar X, \qquad \bar X := \Cmat\, \bar W\, \Cmat^\top.
\end{equation}
Consequently, $\lambda_{\max}(X_\infty(K)) \le \Psi(K)\,\lambda_{\max}(\bar X)$. If~$\Psi(K_1) \le \Psi(K_2)$, then~$K_1$'s proxy upper bound is no larger than~$K_2$'s, and substituting into~\eqref{eq:fail-bound} yields a failure-probability upper bound for~$K_1$ that is no larger than the corresponding upper bound for~$K_2$.
\end{theorem}

\begin{proof}
By~(i), left- and right-multiplying $\Sigma_K^{\mathrm{roll}}\psd l(K)\bar\Sigma$ by $\Bmat_K$ and its transpose and applying~(ii) gives $\Bmat_K\Sigma_K^{\mathrm{roll}}\Bmat_K^\top \psd l(K)b(K)\,\bar W$. The map $X\mapsto\Amat_K^s X(\Amat_K^\top)^s$ preserves $\psd$ for $s\ge 0$, so $\Amat_K^s\Bmat_K\Sigma_K^{\mathrm{roll}}\Bmat_K^\top(\Amat_K^\top)^s\psd l(K)b(K)\Amat_K^s\bar W(\Amat_K^\top)^s$. Iterating~(iii) yields $\Amat_K^s\bar W(\Amat_K^\top)^s\psd\rho_*(K)^{2s}\bar W$, and the geometric sum gives $S_\infty(K)\psd l(K)b(K)\sum_{s\ge 0}\rho_*(K)^{2s}\bar W = \Psi(K)\bar W$. Projecting through $\Cmat$ yields~\eqref{eq:Xinfty-bound}, and operator monotonicity of $\lambda_{\max}$ followed by substitution into~\eqref{eq:fail-bound} gives the failure-bound monotonicity.
\end{proof}

The index $\Psi(K)$ assembles three independent penalties: a large label difficulty $l(K)$, a large injection $b(K)$, and slow contraction (large $\rho_*(K)$). A gain can remain optimal even when its label difficulty is the highest of the four candidates, provided the injection and contraction terms are correspondingly smaller. This three-way competition is precisely what makes the BC gain-tuning problem nontrivial, and Theorem~\ref{thm:scalar} reduces the resolution of that competition to a single scalar comparison. Importantly, the three penalties are not symmetric in the gains: the injection $b(K)$ is typically a strictly increasing function of $K_p$ (stiffer gains inject action errors more aggressively), while the contraction $\rho_*(K)$ is typically a strictly decreasing function of $K_d$ (stronger damping tightens the Lyapunov rate). The label-difficulty term $l(K)$ is the only component with a qualitatively data-driven behavior: harder labels are induced by compliant settings because the policy must compensate for richer closed-loop trajectories.

\begin{corollary}[Finite-horizon convergence]\label{cor:finite}
Under Definition~\ref{def:shape}, the truncation error of the proxy admits the upper bound
\begin{equation}\label{eq:proxy-conv}
X_\infty(K) - X_t(K) \psd \frac{\rho_*(K)^{2t}}{1-\rho_*(K)^2}\,b(K)\,l(K)\,\bar X,
\end{equation}
and consequently $\opnorm{X_\infty(K) - X_t(K)} \le \rho_*(K)^{2t}\,\Psi(K)\,\opnorm{\bar X}$, so the approximation error decays geometrically at rate~$\rho_*(K)^2$ per step.
\end{corollary}

\begin{proof}
$S_\infty(K) - S_t(K) = \sum_{s=t}^\infty \Amat_K^s W_K (\Amat_K^\top)^s$ with~$W_K = \Bmat_K \Sigma_K^{\mathrm{roll}} \Bmat_K^\top$. Repeating the upper-bound chain from the proof of Theorem~\ref{thm:scalar} on the tail sum gives $X_\infty(K) - X_t(K) \psd l(K)b(K) \rho_*(K)^{2t}/(1-\rho_*(K)^2)\bar X$. The operator-norm bound follows since~$\opnorm{\cdot}$ is monotone under~$\psd$ for PSD matrices.
\end{proof}

\subsection{Regime Ordering}

We parameterize gains as~$K = (\alpha, \beta)$ where~$\alpha$ denotes stiffness ($K_p$ scale) and~$\beta$ denotes damping ($K_d$ scale), and define four canonical regimes
\begin{equation}\label{eq:regimes}
\begin{aligned}
\text{CO} &= (\alpha_L, \beta_H), &\quad \text{SO} &= (\alpha_H, \beta_H), \\
\text{CU} &= (\alpha_L, \beta_L), &\quad \text{SU} &= (\alpha_H, \beta_L),
\end{aligned}
\end{equation}
with~$\alpha_H > \alpha_L$ and~$\beta_H > \beta_L$.

\begin{theorem}[Regime ordering]\label{thm:regime}
Suppose~$\Psi(\alpha, \beta)$ is non-decreasing in~$\alpha$ and non-increasing in~$\beta$. Then
\begin{align}
\Psi(\mathrm{CO}) &\le \min\!\left\{\Psi(\mathrm{SO}),\, \Psi(\mathrm{CU})\right\}, \label{eq:co-best}\\
\Psi(\mathrm{SU}) &\ge \max\!\left\{\Psi(\mathrm{SO}),\, \Psi(\mathrm{CU})\right\}, \label{eq:su-worst}
\end{align}
so CO admits the tightest failure-bound upper bound and SU the loosest. The SO vs.\ CU comparison reduces to $\Psi(\alpha_H,\beta_H) \lessgtr \Psi(\alpha_L,\beta_L)$ and is therefore system-dependent.
\end{theorem}

\begin{proof}
Direct from monotonicity. $\Psi(\alpha_L,\beta_H) \le \Psi(\alpha_H,\beta_H)$ by stiffness-monotonicity, and $\Psi(\alpha_L,\beta_H) \le \Psi(\alpha_L,\beta_L)$ by damping-antimonotonicity. Similarly $\Psi(\alpha_H,\beta_H) \le \Psi(\alpha_H,\beta_L)$ and $\Psi(\alpha_L,\beta_L) \le \Psi(\alpha_H,\beta_L)$. The SO-CU comparison involves opposing effects and cannot be resolved without additional information.
\end{proof}

\begin{figure}[t]
\centering
\includegraphics[width=0.7\columnwidth]{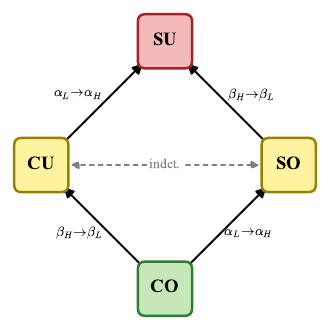}
\caption{Hasse diagram of the regime ordering established by Theorem~\ref{thm:regime}. Solid arrows indicate $\Psi(K_1) \le \Psi(K_2)$. CO is the unique minimum and SU the unique maximum. The dashed line between SO and CU indicates that their ordering is system-dependent.}
\label{fig:hasse}
\end{figure}

\subsection{Verification for Canonical Second-Order Systems}\label{sec:canonical}

To specialize Theorem~\ref{thm:regime}, it suffices to verify one canonical fact: the gain-dependent amplification increases with stiffness and decreases with damping. For the scalar second-order PD system, this fact can be checked in closed form.

\begin{lemma}[Canonical stationary covariance]\label{lem:cont-cov}
Consider the scalar ($n=1$) continuous-time error dynamics~\eqref{eq:cont-state} with mass~$m>0$, gains $K_p=\alpha>0$ and $K_d=\beta>0$, and white action-error intensity~$\sigma^2$. The stationary covariance of~$\xvec_t$ is the unique PSD solution of
\begin{equation*}
\Amat_K^{\mathrm{c}}P + P(\Amat_K^{\mathrm{c}})^\top + \Bmat_K^{\mathrm{c}}\sigma^2(\Bmat_K^{\mathrm{c}})^\top = 0,
\end{equation*}
and is given by
\begin{equation}\label{eq:Pc-cont}
P^{\mathrm{c}}(\alpha,\beta) = \begin{bmatrix} \dfrac{\sigma^2\alpha}{2\beta} & 0 \\ 0 & \dfrac{\sigma^2\alpha^2}{2\beta\,m} \end{bmatrix}.
\end{equation}
Thus the stationary position-error variance is
\begin{equation}\label{eq:Xc-cont}
X_\infty^{\mathrm{c}}(\alpha,\beta) = \Cmat P^{\mathrm{c}}\Cmat^\top = \frac{\sigma^2\alpha}{2\beta}.
\end{equation}
\end{lemma}

\begin{proof}
With $A_c = \begin{bsmallmatrix}0 & 1 \\ -\alpha/m & -\beta/m\end{bsmallmatrix}$, $B_c=(\alpha/m)e_2$, and $P=\begin{bsmallmatrix}p_{11}&p_{12}\\p_{12}&p_{22}\end{bsmallmatrix}$, the Lyapunov equation gives $2p_{12}=0$, $p_{22}=(\alpha/m)p_{11}$, and $2(\beta/m)p_{22}=(\alpha/m)^2\sigma^2$. Solving yields~\eqref{eq:Pc-cont}. Since $\alpha,\beta>0$ make $A_c$ Hurwitz, this PSD solution is unique.
\end{proof}

\begin{theorem}[Global monotonicity in canonical systems]\label{thm:canonical}
The canonical stationary variance $X_\infty^{\mathrm{c}}(\alpha,\beta)=\sigma^2\alpha/(2\beta)$ is strictly increasing in stiffness~$\alpha$ and strictly decreasing in damping~$\beta$ on the entire stable orthant $\{(\alpha,\beta):\alpha,\beta>0\}$, covering both underdamped ($\beta^2<4m\alpha$) and overdamped ($\beta^2\ge 4m\alpha$) regimes. Consequently, Theorem~\ref{thm:regime} holds for the canonical system with $\Psi^{\mathrm{c}}(\alpha,\beta):=X_\infty^{\mathrm{c}}(\alpha,\beta)/\bar X^{\mathrm{c}}$ for any positive scalar reference~$\bar X^{\mathrm{c}}$.
\end{theorem}

\begin{proof}
By~\eqref{eq:Xc-cont},
\[
\partial_\alpha X_\infty^{\mathrm{c}} = \frac{\sigma^2}{2\beta}>0,
\qquad
\partial_\beta X_\infty^{\mathrm{c}} = -\frac{\sigma^2\alpha}{2\beta^2}<0.
\]
The stability condition for the scalar second-order system is exactly $\alpha,\beta>0$, so no underdamped/overdamped case split is required.
\end{proof}

Equation~\eqref{eq:Xc-cont} is the continuous-time stationary-variance term used by the attenuation explanation of~\cite{bronars2026tune}. Here it serves as the canonical instance of the amplification proxy in Theorem~\ref{thm:failure}. After normalization, $\Psi^{\mathrm{c}}\propto\alpha/\beta$, so CO is the canonical minimum, SU the canonical maximum, and SO versus CU is decided by the corresponding stiffness--damping ratios.

\begin{proposition}[ZOH leading-order consistency]\label{prop:zoh-inherit}
Let $X_\infty^{\mathrm{d}}(\alpha,\beta,\Delta t)$ denote the stationary position-error proxy of the exact ZOH discretization~\eqref{eq:zoh}. On any compact gain set contained in $\{(\alpha,\beta):\alpha,\beta>0\}$,
\begin{equation}\label{eq:zoh-leading}
\begin{aligned}
X_\infty^{\mathrm{d}}(\alpha,\beta,\Delta t)
&= \Delta t\,X_\infty^{\mathrm{c}}(\alpha,\beta) + O(\Delta t^2) \\
&= \Delta t\,\frac{\sigma^2\alpha}{2\beta}+O(\Delta t^2),
\end{aligned}
\end{equation}
uniformly as $\Delta t\to0^+$. Hence, for sufficiently small control periods on that gain set, the exact ZOH proxy has the same monotonicity in~$\alpha$ and~$\beta$ as the continuous-time proxy.
\end{proposition}

\begin{proof}
The ZOH matrices satisfy $\Amat_K=I+\Amat_K^{\mathrm{c}}\Delta t+O(\Delta t^2)$ and $\Bmat_K=\Delta t\,\Bmat_K^{\mathrm{c}}+O(\Delta t^2)$ uniformly on compact stable gain sets. Substituting these expansions into the discrete Lyapunov equation~\eqref{eq:lyap} shows that its solution has the form $S_\infty^{\mathrm{d}}=\Delta t\,P^{\mathrm{c}}+O(\Delta t^2)$, because the continuous-time Lyapunov operator is invertible for Hurwitz~$\Amat_K^{\mathrm{c}}$. Projection by~$\Cmat$ gives~\eqref{eq:zoh-leading}. Differentiating the expansion gives $\partial_\alpha X_\infty^{\mathrm{d}}=\Delta t\,\sigma^2/(2\beta)+O(\Delta t^2)$ and $\partial_\beta X_\infty^{\mathrm{d}}=-\Delta t\,\sigma^2\alpha/(2\beta^2)+O(\Delta t^2)$, so the signs persist for small enough~$\Delta t$ on compact gain ranges.
\end{proof}

\section{Numerical Illustration}\label{sec:illustration}

\begin{figure*}[t]
\centering
\includegraphics[width=\textwidth]{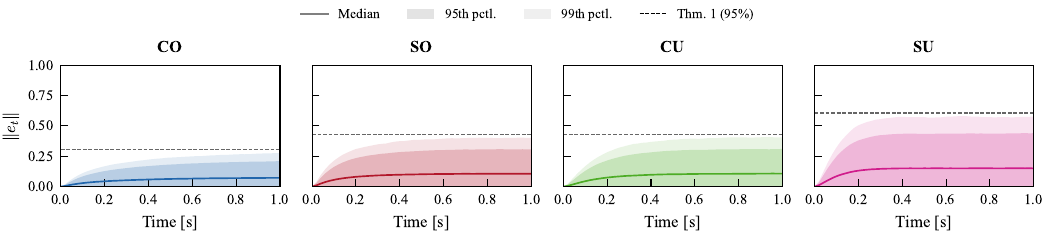}
\caption{Monte Carlo position-error envelopes ($N=50{,}000$ rollouts) for the four gain regimes under exact ZOH discretization. Shaded bands show the 95th and 99th percentiles. The dashed line is the steady-state 95th-percentile bound from Theorem~\ref{thm:subgaussian}. All panels share the same vertical scale. CO yields the tightest envelopes, confirming the predicted ordering.}
\label{fig:envelopes}
\end{figure*}

\begin{figure}[t]
\centering
\includegraphics[width=\columnwidth]{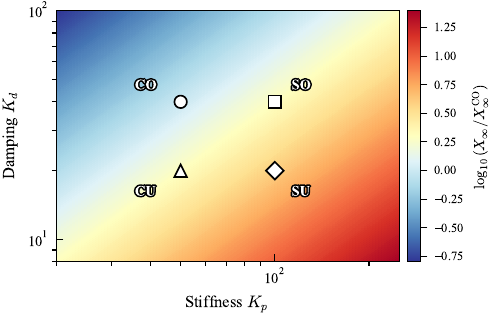}
\caption{Stationary proxy~$X_\infty(K)$ (normalized by~$X_\infty^{\mathrm{CO}}$, log scale) over the gain parameter space, computed via exact ZOH discretization at $\Delta t = 0.02$\,s. Markers locate the four canonical regimes. The landscape is monotone increasing in~$K_p$ and decreasing in~$K_d$ over the entire stable region, consistent with Theorem~\ref{thm:canonical} and the leading-order ZOH relation in Proposition~\ref{prop:zoh-inherit}.}
\label{fig:heatmap}
\end{figure}

\begin{figure}[t]
\centering
\includegraphics[width=\columnwidth]{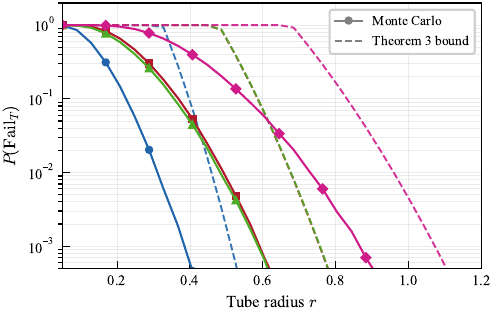}
\caption{Empirical failure rate $\hat P(\mathrm{Fail}_T)$ (solid) versus the Theorem~\ref{thm:failure} upper bound (dashed) as a function of the success-tube radius~$r$ for the four canonical regimes. The bound dominates the Monte Carlo curve in every regime, and the regime ordering CO\,$\prec$\,SO\,$\approx$\,CU\,$\prec$\,SU is preserved at all $r$.}
\label{fig:failure}
\end{figure}

We illustrate the framework on the canonical scalar second-order system, verifying sub-Gaussian propagation, the failure bound, and the regime ordering. The closed-loop matrices are obtained by exact ZOH sampling~\eqref{eq:zoh}, whose leading-order continuous-time limit is described by Proposition~\ref{prop:zoh-inherit}.

\subsection{Setup}

Consider unit mass ($m=1$\,kg) under PD control sampled at $\Delta t=0.02$\,s ($50$\,Hz). We evaluate the four canonical gain regimes of~\eqref{eq:regimes} with $\alpha_L=50$, $\alpha_H=100$, $\beta_L=20$, $\beta_H=40$. Action prediction errors are drawn i.i.d.\ from $\mathcal N(0,1)$, satisfying~\eqref{eq:subgauss} with $\Sigma_K^{\mathrm{roll}}=1$. For each regime we simulate $N=50{,}000$ independent rollouts of $T=50$ steps ($1$\,s task horizon), and compute the stationary proxy $X_\infty(K)$ from the discrete Lyapunov equation~\eqref{eq:lyap}.

\begin{table}[t]
\centering
\caption{System quantities and empirical failure rates for the scalar canonical system. $X_\infty^{\mathrm{c}}$ from~\eqref{eq:Xc-cont}, $X_\infty^{\mathrm{d}}$ from the ZOH discrete Lyapunov.}
\label{tab:illustration}
\renewcommand{\arraystretch}{1.15}
\begin{tabular}{@{}lcccc@{}}
\toprule
& \textbf{CO} & \textbf{SO} & \textbf{CU} & \textbf{SU} \\
\midrule
$K_p$ & 50 & 100 & 50 & 100 \\
$K_d$ & 40 & 40 & 20 & 20 \\
$\rho(A_K)$ & 0.974 & 0.948 & 0.943 & 0.819 \\
$X_\infty^{\mathrm{c}} = \sigma^2 K_p/(2 K_d)$ & 0.625 & 1.250 & 1.250 & 2.500 \\
$X_\infty^{\mathrm{d}}$ (ZOH discrete) & 0.012 & 0.025 & 0.025 & 0.050 \\
$X_\infty^{\mathrm{d}} / X_\infty^{\mathrm{d,CO}}$ & \textbf{1.0} & 2.0 & 2.0 & \textbf{4.0} \\
$\hat P(\mathrm{Fail}_T)$ & 0.01 & 0.26 & 0.22 & 0.75 \\
\bottomrule
\end{tabular}
\end{table}

\subsection{Sub-Gaussian Error Propagation}

Fig.~\ref{fig:envelopes} shows the position-error magnitude $\|\evec_t\|$ across regimes, with median, $95$th-percentile, and $99$th-percentile envelopes from the Monte Carlo ensemble. The dashed line is the steady-state $95$th-percentile threshold predicted by Theorem~\ref{thm:subgaussian} via $r_{95}=\sqrt{2X_\infty(K)\ln 40}$, derived from~\eqref{eq:et-tail} with $n=1$. The empirical $95$th percentile stays inside the predicted bound in every regime, confirming sub-Gaussian propagation. The error magnitudes differ systematically: CO yields steady-state errors roughly half those of SO or CU and four times smaller than SU, despite identical action-error distributions across the four regimes.

\subsection{Failure Bound and Regime Ordering}

Table~\ref{tab:illustration} reports both the closed-form continuous-time stationary variance $X_\infty^{\mathrm{c}}=\sigma^2\alpha/(2\beta)$ and the exact discrete ZOH proxy $X_\infty^{\mathrm{d}}$. Both quantities respect the predicted ordering $\mathrm{CO}\prec\mathrm{SO}\approx\mathrm{CU}\prec\mathrm{SU}$, with the near-equality of SO and CU exemplifying the system-dependent SO/CU comparison. The Monte Carlo failure rates corroborate the ordering: at $r=0.3$, CO fails on $1\%$ of rollouts, SO and CU on $26\%$ and $22\%$, and SU on $75\%$. Notably CO has a \emph{larger} spectral radius ($0.974$) than SU ($0.819$) yet achieves a fourfold smaller proxy because its injection is proportionally weaker, reflecting the three-way competition between label difficulty, injection strength, and contraction inside $\Psi(K)$.

Fig.~\ref{fig:heatmap} sweeps $X_\infty(K)$ over the $(K_p,K_d)$ parameter space. The landscape is monotone increasing in stiffness and decreasing in damping over the \emph{entire} sampled region, covering both underdamped and overdamped regimes, with CO at the global minimum and SU at the maximum, consistent with Theorem~\ref{thm:canonical}.

\subsection{Test of the Failure Bound}

Fig.~\ref{fig:failure} compares the empirical failure rate $\hat P(\mathrm{Fail}_T)$ with the upper bound of Theorem~\ref{thm:failure} as the success-tube radius~$r$ varies. For each regime, the dashed curve plots $\min\{1,\,2(T{+}1)\exp(-r^2/(2\Gamma_T(K)L_{\mathrm{roll}}(K)))\}$ from Corollary~\ref{cor:euclidean} with $L_{\mathrm{roll}}=1$ and the empirical $\Gamma_T(K)$ from~\eqref{eq:gamma}. The bound dominates the Monte Carlo curve uniformly across regimes and tube radii, with CO retaining the largest gap (well-conditioned tail) and SU sitting closest to its bound, and the relative ordering matches Theorem~\ref{thm:regime} at every~$r$. The failure-bound exponent is therefore qualitatively correct as a comparator and quantitatively conservative for absolute prediction. The sub-Gaussian assumption is the dominant source of looseness, since $L_{\mathrm{roll}}$ replaces the per-coordinate proxy by its trace and the union over $T+1$ steps adds an unconditional log-factor.

\subsection{Robosuite \texttt{PickPlaceCan} Evaluation}\label{sec:robosuite}

We use the robosuite~\cite{robosuite2020} \texttt{PickPlaceCan} task as a contact-rich check of the gain-dependent mechanism with learned policies rather than direct additive action noise. The setup uses a Panda robot, low-dimensional observations, joint-position control at $20$\,Hz, and a horizon of $150$ steps. For each gain setting $K=(k_p,\zeta)$, with surrogate damping $k_d=2\zeta\sqrt{k_p}$, the same expert transitions are relabeled by the one-step inverse model
\begin{equation}\label{eq:robosuite-relabel}
\tilde a_{K,t}
=
\argmin_a
\left\|
\begin{bmatrix}q_{t+1}\\ \dot q_{t+1}\end{bmatrix}
-
f_K(q_t,\dot q_t,a)
\right\|_2^2,
\end{equation}
where $f_K$ is the ZOH second-order joint model used in the analysis. Thus the expert state distribution is shared across regimes, and only the gain-dependent arm labels and controller gains differ. For every seed and regime we train a separate two-hidden-layer MLP behavior-cloning policy on inverse-PD joint-delta labels, using joint state, end-effector pose, gripper state, can pose, and normalized time as inputs. The policy is then deployed online in robosuite under the corresponding fixed-impedance joint-position controller. The measured peak errors and failures below come from these closed-loop neural-policy rollouts, not from injected action noise.

\begin{table}[!h]
\centering
\caption{Robosuite \texttt{PickPlaceCan} regimes and mean online trained-policy results over three seeds. Each row averages separately trained MLP deployments. Failure is measured with tube radius $r=4.6$ and $C_{\rm rs}=(\widehat S_T^{\rm pulse})^2\widehat L_{\rm cal}$.}
\label{tab:robosuite}
\scriptsize
\setlength{\tabcolsep}{3.0pt}
\renewcommand{\arraystretch}{1.12}
\begin{tabular}{@{}lcccccc@{}}
\toprule
\textbf{Regime} & $k_p$ & $\zeta$ & $\widehat L_{\rm val}$ & $C_{\rm rs}$ & $Z_K$ & \textbf{Fail} \\
\midrule
CO & 50  & 2.0 & 0.679 & $1.5{\times}10^5$ & 2.75 & 0.000 \\
SO & 200 & 2.0 & 0.382 & $1.7{\times}10^6$ & 3.12 & 0.013 \\
CU & 50  & 0.5 & 0.714 & $2.4{\times}10^6$ & 3.98 & 0.213 \\
SU & 200 & 0.5 & 0.270 & $1.4{\times}10^8$ & 6.00 & 0.893 \\
\bottomrule
\end{tabular}
\end{table}

The result data consist of one row for each seed--regime pair, together with calibration-budget curves. Each row corresponds to one trained network and its held-out online deployments. We report the expert-state validation loss $\widehat L_{\rm val}$, a rollout-local calibration loss $\widehat L_{\rm cal}$, the peak joint error $Z_K=\max_t\|q_t-q_t^\star\|_2$, and tube failure. After training, we estimate the closed-loop amplification term by applying a small one-step action pulse to the deployed policy at matched rollout states. The normalized peak deviation defines $(\widehat S_T^{\rm pulse})^2$. This pulse is a post-hoc local response measurement, not the source of the online failure data. The empirical comparator is
\begin{equation}\label{eq:robosuite-product}
C_{\rm rs}(K)
=
(\widehat S_T^{\rm pulse}(K))^2\widehat L_{\rm cal}(K).
\end{equation}

\begin{figure}[t]
\centering
\includegraphics[width=0.49\textwidth]{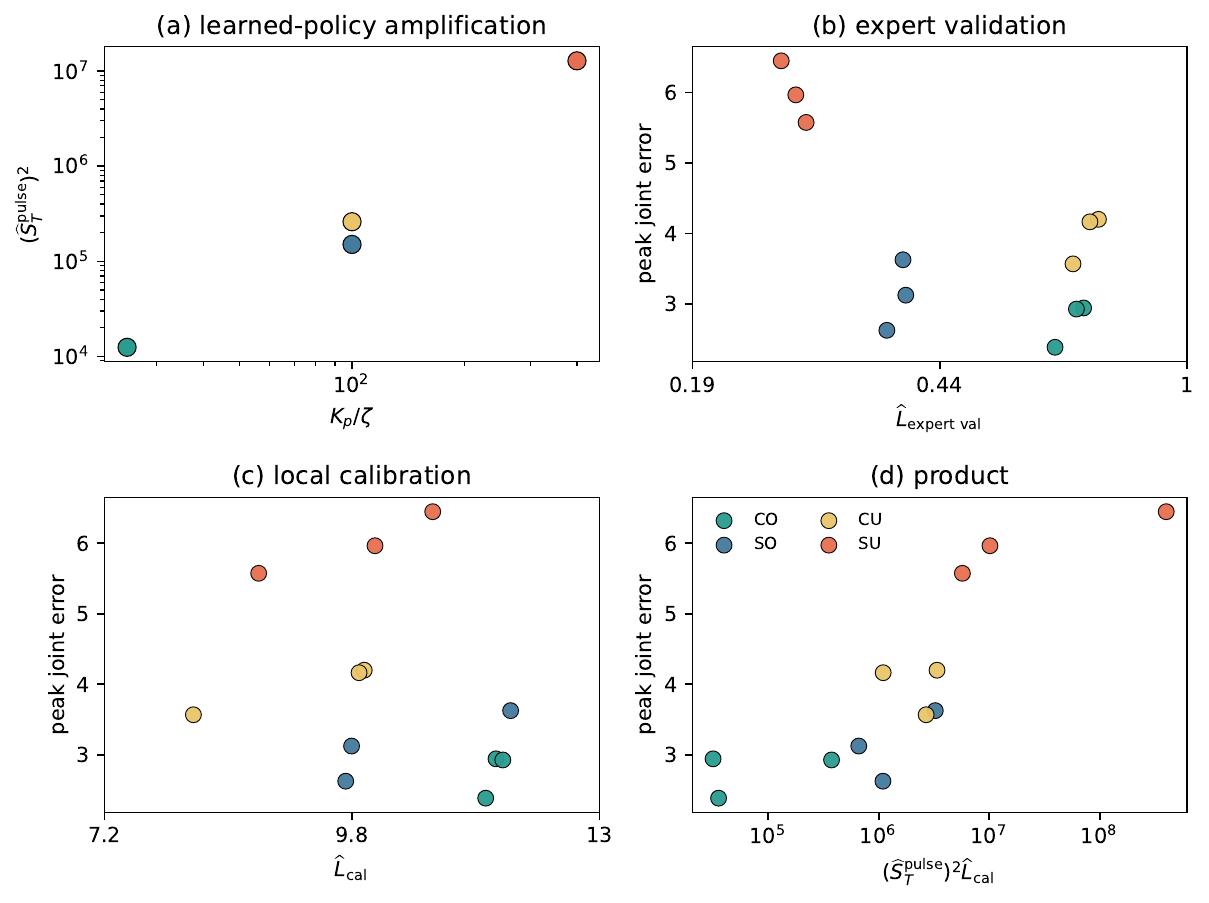}
\caption{Online robosuite \texttt{PickPlaceCan} deployments of trained MLP policies. Panel (a) shows post-training local amplification estimates, and panels (b)--(d) show seed-level peak joint error from neural-policy rollouts against expert validation loss, rollout-local calibration loss, and the product $C_{\rm rs}$. Validation loss alone favors the stiff-underdamped regime, whereas $C_{\rm rs}$ separates the high-error SU regime from the compliant-overdamped regime.}
\label{fig:robosuite_can_online}
\end{figure}

Table~\ref{tab:robosuite} and Fig.~\ref{fig:robosuite_can_online} show that expert validation loss is not a reliable closed-loop predictor for the trained neural policies: SU attains the smallest $\widehat L_{\rm val}$ but has the largest peak error and tube-failure rate, while CO has higher validation loss but no observed tube failures. The local calibration loss alone also misses the ordering, because the dominant variation comes from gain-dependent amplification. Multiplying by the post-training response proxy recovers the qualitative pattern predicted by Theorem~\ref{thm:failure}: CO is the safest regime, SU is the most fragile, and the SO--CU middle ordering is system-dependent. Since the rollouts include contact events, $C_{\rm rs}$ should be read as a ranking statistic rather than a calibrated probability.

\section{Discussion}\label{sec:discussion}

\subsection{Gain-Selection Criterion}

Theorem~\ref{thm:failure} turns the attenuation-difficulty tradeoff of~\cite{bronars2026tune} into a finite-horizon decision rule. Rather than minimizing the training or validation loss, practitioners should minimize the product $\Gamma_T(K)\cdot(\hat L_{\mathrm{va}}(K)+\varepsilon_{\mathrm{gen}}(K,\delta))$, equivalently $\Psi(K)$ when shape-preserving structure holds or $X_\infty^{\mathrm{c}}(K)$ for canonical second-order systems. Theorems~\ref{thm:regime} and~\ref{thm:canonical} recommend the compliant-overdamped regime in agreement with~\cite{bronars2026tune}. In our bound, the higher-validation-MSE yet better-closed-loop-performance pattern appears inside the product: CO incurs larger label difficulty $l(K)$ but a proportionally smaller amplification, and the exponent of~\eqref{eq:fail-bound} rewards that trade. A practical workflow follows directly. Start from compliant-overdamped gains, monitor $\Gamma_T(K)\cdot\hat L_{\mathrm{va}}(K)$ rather than the training loss alone, and resolve the SO versus CU choice empirically because the ordering is fundamentally system-dependent.

The amplification term can also be estimated empirically when an accurate linear model is unavailable. Small action pulses around matched rollout states give a local estimate of the closed-loop response, which can be multiplied by validation or calibration loss to rank gains. This empirical product should be interpreted as an ordering statistic rather than as an absolute probability, but it preserves the main lesson: compare gains after accounting for how the controller filters action residuals. The same decomposition points to interventions: improve labels or calibration, soften injection, or increase damping-driven contraction rather than relying on lower training loss alone. Since all three enter the same product, model-based proxies and empirical pulse estimates remain comparable even when the latter are used only for qualitative ranking.

\subsection{Relation to Classical Analysis}

Classical compounding-error analysis~\cite{ross2010efficient} bounds the performance gap by $O(\varepsilon T^2)$. Our amplification index refines the $T^2$ factor. For Schur-stable systems, $\Gamma_T(K) = O(1/(1-\rho(\Amat_K))^2)$ remains bounded as $T\to\infty$ and converges to the squared $\mathcal H_2$ norm $\lim_{T\to\infty}\Gamma_T(K) = \|\mathcal G_K\|_{\mathcal H_2}^2$~\cite{recht2019tour}, connecting the analysis to classical robust control. Block~et~al.~\cite{block2024butterfly} quantify how training noise compounds during autoregressive rollout, while our amplification index complements that line by exposing how the controller gains modulate the same pathway through $A_K$ and $B_K$. The indeterminacy of the SO versus CU comparison matches the system-dependent observations reported in~\cite{bronars2026tune}.

This separation is useful in gain sweeps: supervised metrics measure residual size on a chosen label distribution, whereas $\Gamma_T(K)$ measures the closed-loop filter applied to those residuals. Two policies with similar validation loss can therefore have different task-level tails, and two gains with different validation loss can be correctly ranked only after including amplification.

\subsection{Extensions and Limitations}

The linearization about the expert trajectory is the main simplification. Nonlinear, contact-rich settings are a natural next step via contraction theory~\cite{boffi2021regret}. The independence assumption on $\xivec_t$ weakens once covariate shift accumulates, and in that regime the bounds remain valid after replacing $\Sigma_K^{\mathrm{roll}}$ by an upper bound on the conditional sub-Gaussian proxy. Multi-step action chunks~\cite{chi2025diffusion, zhao2023learning} enter the framework unchanged by treating a length-$H$ chunk as a single $Hm$-dimensional prediction error and evaluating $\Gamma_T(K)$ at the chunk-execution rate. For multi-joint systems with non-uniform gains the matrix-level comparison through $X_\infty(K)$ replaces the scalar reduction whenever the shape-preserving structure of Definition~\ref{def:shape} fails to hold uniformly across joints.

\section{Conclusion}\label{sec:conclusion}

We quantified how PD gains shape BC failure through sub-Gaussian propagation, a validation-loss-based failure bound, a scalar upper-bound ordering, and global canonical monotonicity across underdamped and overdamped regimes. The stationary variance $X_\infty^{\mathrm{c}}=\sigma^2 K_p/(2K_d)$ recovers the attenuation mechanism of~\cite{bronars2026tune} and embeds it in finite-horizon failure bounds. Future work should tighten the discrete-time bound and extend the analysis to nonlinear contact dynamics and joint gain-policy optimization.

\bibliography{references}

\end{document}